%% file: VSSGP-arXiv.tex
\documentclass{article}

\usepackage{times}
\usepackage{graphicx} 
\usepackage[caption=false]{subfig}

\usepackage{natbib}

\usepackage{algorithm}
\usepackage{algorithmic}

\usepackage[accepted]{icml2015arxiv}

\usepackage[utf8]{inputenc}
\usepackage{amsthm}
\usepackage{amsmath}
\usepackage{amsfonts}
\usepackage{amssymb}
\usepackage{dsfont}
\usepackage{color}
\allowdisplaybreaks
\newcommand{\R}{\mathbb{R}}
\newcommand{\N}{\mathcal{N}}
\newcommand{\cL}{\mathcal{L}}
\newcommand{\cO}{\mathcal{O}}
\newcommand{\svert}{~|~}
\newcommand{\td}{\text{d}}

\newcommand{\x}{\mathbf{x}}
\newcommand{\bx}{\overline{\x}}
\newcommand{\bb}{\overline{b}}
\newcommand{\y}{\mathbf{y}}
\newcommand{\z}{\mathbf{z}}
\newcommand{\w}{\mathbf{w}}
\newcommand{\ba}{\mathbf{a}}
\newcommand{\m}{\mathbf{m}}

\newcommand{\bL}{\mathbf{L}}
\newcommand{\A}{\mathbf{A}}
\newcommand{\X}{\mathbf{X}}
\newcommand{\Y}{\mathbf{Y}}
\newcommand{\F}{\mathbf{F}}
\newcommand{\I}{\mathbf{I}}
\newcommand{\M}{\mathbf{M}}
\newcommand{\p}{\mathbf{p}}
\newcommand{\bp}{\overline{\p}}
\newcommand{\bz}{\mathbf{0}}

\newcommand{\s}{\mathbf{s}}
\newcommand{\Unif}{\text{Unif}}
\newcommand{\bo}{\text{\boldmath$\omega$}}
\newcommand{\bSigma}{\text{\boldmath$\Sigma$}}

\newcommand{\Kh}{\widehat{K}}
\newcommand{\Cov}{\text{Cov}}
\newcommand{\Var}{\text{Var}}
\newcommand{\tr}{\text{tr}}

\newcommand{\diag}{\text{diag}}
\newcommand{\KL}{\text{KL}}

\newtheorem{identity}{Identity}
\newtheorem{proposition}{Proposition}
\theoremstyle{definition}
\newtheorem{discussion}{Discussion}

\def\fl[#1\]{\begin{align}#1\end{align}}
\def\[#1\]{\begin{align*}#1\end{align*}}

\def\*[#1\]{\begin{align*}#1\end{align*}}

\usepackage[small,compact]{titlesec}
\titlespacing\section{0pt}{0pt plus 1pt minus 0pt}{0pt plus 0pt minus 0pt}
\titlespacing\subsection{0pt}{0pt plus 1pt minus 0pt}{0pt plus 0pt minus 0pt}
\titlespacing\subsubsection{0pt}{0pt plus 1pt minus 0pt}{0pt plus 0pt minus 0pt}
\expandafter\def\expandafter\normalsize\expandafter{%
    \normalsize
    \setlength\abovedisplayskip{2pt}
    \setlength\belowdisplayskip{2pt}
    \setlength\abovedisplayshortskip{0pt}
    \setlength\belowdisplayshortskip{0pt}
}

\usepackage{hyperref}

\icmltitlerunning{Improving the Gaussian Process Sparse Spectrum Approximation by Representing Uncertainty in Frequency Inputs}

\begin{document} 

\twocolumn[
\icmltitle{Improving the Gaussian Process Sparse Spectrum Approximation by Representing Uncertainty in Frequency Inputs}

\vspace{-2mm}
\icmlauthor{Yarin Gal}{yg279@cam.ac.uk}
\icmlauthor{Richard Turner}{ret26@cam.ac.uk}
\icmladdress{University of Cambridge}


\vskip 0.in
]

\input{VSSGP-paper}
\include{VSSGP-appendix}
\end{document}

%% file: VSSGP-paper.tex
\begin{abstract} 
Standard sparse pseudo-input approximations to the Gaussian process (GP) cannot handle complex functions well. Sparse spectrum alternatives attempt to answer this but are known to over-fit. We suggest the use of variational inference for the sparse spectrum approximation to avoid both issues. We model the covariance function with a finite Fourier series approximation and treat it as a random variable. The random covariance function has a posterior, on which a variational distribution is placed. The variational distribution transforms the random covariance function to fit the data. We study the properties of our approximate inference, compare it to alternative ones, and extend it to the distributed and stochastic domains. Our approximation captures complex functions better than standard approaches and avoids over-fitting.
\end{abstract} 

\vspace{-5mm}
\section{Introduction}

The Gaussian process \citep[GP, ][]{Rasmussen2005Gaussian} is a powerful tool for modelling distributions over non-linear functions. It offers robustness to over-fitting, a principled way to tune hyper-parameters, and uncertainty bounds over the outputs. 
These properties are critical for tasks including non-linear function regression, reinforcement learning, density estimation, and more \citep{brochu2010tutorial,rasmussen2003gaussian,engel2005reinforcement,titsias2010bayesian}.
But the advantages of the Gaussian process come with a great computational cost. Evaluating the GP posterior involves a large matrix inversion -- for $N$ data points the model requires $\cO(N^3)$ time complexity. 

Many approximations to the GP have been proposed to reduce the model's time complexity. \citet{Quinonero-candela05unifying} survey approaches relying on \textit{``sparse pseudo-input''} approximations. In these, a small number of points in the input space with corresponding outputs (``inducing inputs and outputs'') are used to define a new Gaussian process. The new GP is desired to be as close as possible to the GP defined on the entire dataset, and the matrix inversion is now done with respect to the inducing points alone. 
These approaches are suitable for locally complex functions. The approximate model would place most of the inducing points in regions where the function is complex, and only a small number of points would be placed in regions where the function is not. Highly complex functions cannot be modelled well with this approach.

\citet{lazaro2010sparse} suggested an alternative approximation to the GP model. In their paper they suggest the decomposition of the GP's stationary covariance function into its Fourier series. The infinite series is then approximated with a finite one. They optimise over the frequencies of the series to minimise some divergence from the full Gaussian process. This approach was named a \textit{``sparse spectrum''} approximation.
This approach is closely related to the one suggested by \citet{rahimi2007random} in the randomised methods community (random projections). In \citet{rahimi2007random}'s approach, the frequencies are randomised (sampled from some distribution rather than optimised) and the Fourier coefficients are computed analytically. 
Both approaches capture globally complex behaviour, but the direct optimisation of the different quantities often leads to some form of over-fitting (as reported in \citep{wilson2014fast} for the SSGP and shown below for random projections).
Similar over-fitting problems observed with the \textit{sparse pseudo-input} approximation were answered with variational inference \citep{Titsias2009Variational}. 

We suggest the use of variational inference for the \textit{sparse spectrum} approximation. 
This allows us to avoid over-fitting while efficiently capturing globally complex behaviour.
We replace the stationary covariance function with a finite approximation obtained from Monte Carlo integration.
This finite approximation is a random variable, and conditioned on a dataset this random variable has an intractable posterior. 
We approximate this posterior with variational inference, resulting in a non-stationary finite rank covariance function. 
The approximating variational distribution transforms the covariance function to fit the data well. The prior from the GP model keeps the approximating distribution from over-fitting to the data.

Like in \citep{lazaro2010sparse}, we can marginalise over the Fourier coefficients. This results in approximate inference with $\cO(NK^2 + K^3)$ time complexity with $N$ data points and $K$ inducing frequencies (components in the Fourier expansion). 
This is the same as that of the sparse pseudo-input and sparse spectrum approximations. 
We can further optimise a variational distribution over the frequencies reducing the time complexity to $\cO(NK^2)$. This factorises the lower bound and allows us to perform distributed inference, resulting in $\cO(K)$ time complexity given a sufficient number of nodes in a distributed framework. 
We can approximate the latter lower bound and use random subsets of the data (mini batches) employing stochastic variational inference \citep{Hoffman2013Stochastic}. This results in $\cO(SK^2)$ time complexity with $S << N$ the size of the mini-batch\footnote{Python code for all inference algorithms is available at \url{http://github.com/yaringal/VSSGP}}. 

In the experiments section we demonstrate the properties of our GP approximation and compare it to alternative approximations.
We describe qualitative properties of the approximation and discuss how the approximation can be used to learn the covariance function by fitting to the data. 
We compare the approximation to the full Gaussian process, sparse spectrum GP, sparse pseudo-input GP, and random projections. 
We show that alternative approximations either over-fit or under-fit even on simple datasets. 
We empirically demonstrate the advantages of the variational inference in avoiding over-fitting by comparing the approximation to the sparse spectrum one on audio data from the TIMIT dataset.
We compare the stochastic optimisation to the non-stochastic one, and compare the performance to the sparse pseudo-input SVI. 
Finally, we inspect the model's time accuracy trade-off and show that it avoids over-fitting as the number of parameters increases.

\section{Sparse Spectrum Approximation in Gaussian Process Regression}

We use Bochner’s theorem \citep{bochner1959lectures} to reformulate the covariance function in terms of its frequencies.
Since our covariance function $K(\x, \y)$ is stationary, it can be represented as $K(\x-\y)$ for all $\x,~\y \in \R^{Q}$. Following Bochner’s theorem, $K(\x-\y)$ can be represented as the Fourier transform of some finite measure $\sigma^2 p(\w)$ with $p(\w)$ a probability density,
\fl[
K(\x - \y) &= \int_{\R^Q} \sigma^2 p(\w) e^{- 2 \pi i\w^T(\x-\y)} \td \w \notag\\
&=
\int_{\R^Q} \sigma^2 p(\w) \cos(2 \pi \w^T(\x-\y)) \td \w
\label{eq:bochner}
\]
since the covariance function is real-valued. 

This can be approximated as a finite sum with $K$ terms using Monte Carlo integration,
\[
K(\x - \y) &\approx 
\frac{\sigma^2 }{K} \sum_{k=1}^K 
\cos \big( 2 \pi \w_k^T \big( (\x-\z_k)-(\y-\z_k) \big) \big) 
\]
with $\w_k \sim p(\w)$ and $\z_k$ some $Q$ dimensional vectors for $k = 1, ..., K$. The points $\z_k$ act as inducing inputs, and will have corresponding inducing frequencies in our approximation. For the sparse spectrum GP, these points take value 0. These will be explained in detail in a later section.

Using identity \ref{identity:1} proved in the appendix we rewrite the terms above for every $k$ as
\[
&\cos \big( 2 \pi \w_k^T \big( (\x-\z_k)-(\y-\z_k) \big) \big) 
\notag\\
&\qquad\qquad =
\int_{0}^{2\pi} \frac{1}{2\pi} 
\sqrt{2} \cos \big(2 \pi \w_k^T(\x-\z_k) + b \big) \notag\\
&\qquad\qquad\qquad \cdot 
\sqrt{2} \cos \big(2 \pi \w_k^T(\y-\z_k) + b \big) \td b.
\]

This integral can again be approximated as a finite sum using Monte Carlo integration. To keep the notation simple, we approximate the integral with a single sample\footnote{The above transformation and approximate integration are used in the \textit{randomised methods} literature  \citep[``Random projections'', ][]{rahimi2007random}. It was shown to give better approximation than Monte Carlo integration of eq.\ \ref{eq:bochner}. Intuitively it is equivalent to a random phase shift for each basis function.} for every $k$,
\vspace{-2mm}
\[
K(\x - \y) &\approx \frac{\sigma^2 }{K} \sum_{k=1}^K
\sqrt{2} \cos(2 \pi \w_k^T(\x-\z_k) + b_k)  \notag\\
&\qquad \qquad \cdot 
\sqrt{2} \cos(2 \pi \w_k^T(\y-\z_k) + b_k) \\
&=: \Kh(\x - \y)
\]
with $b_k \sim \Unif[0, 2\pi]$, defining the approximate covariance function $\Kh$. 
We refer to $(\w_k)_{k=1}^K$ as inducing frequencies and to $(b_k)_{k=1}^K$ as phases, and denote $\bo = (\w_k, b_k)_{k=1}^K$.
Note that this integral could be approximated with any arbitrary number of samples instead. 

We denote $\X \in \R^{N \times Q}$ the inputs and $\Y \in \R^{N \times D}$ the outputs of a real-valued dataset with $N$ data points. In Gaussian process regression we find the probability $P(\Y | \X)$ with the assumption that the function generating $\Y$ is drawn from a Gaussian process. The full GP model is defined as (assuming stationary covariance function $K(\cdot, \cdot)$): 
\[
\F \svert \X &\sim \N(\bz, K(\X, \X)) \\
\Y \svert \F &\sim \N(\F, \tau^{-1} \I)
\]
with some precision hyper-parameter $\tau$.

Using $\Kh$ instead as the covariance function of the Gaussian process yields the following generative model: 
\[
\w_k &\sim p(\w), 
~b_k \sim \Unif[0, 2\pi], \notag\\
\bo &= (\w_k, b_k)_{k=1}^K \notag\\
\Kh(\x, \y) &= \frac{\sigma^2 }{K} \sum_{k=1}^K \sqrt{2} \cos \big(2 \pi \w_k^T(\x-\z_k) + b_k \big)  \notag\\
&\qquad \qquad \cdot 
\sqrt{2} \cos \big(2 \pi \w_k^T(\y-\z_k) + b_k \big) \notag\\
\F \svert \X, \bo &\sim \N(\bz, \Kh(\X, \X)) \notag\\
\Y \svert \F &\sim \N(\F, \tau^{-1} \I).
\]

\section{Random Covariance Functions}
$K$ is a deterministic covariance function of its inputs; $\Kh$ is a random finite rank covariance function. As such, we can find the conditional distribution of the covariance function given a dataset (more precisely, the conditional distribution of $\bo$). 
This is a powerful view of this approximation -- it allows us to transform the covariance function to fit the data well, while the prior keeps it from over-fitting to the data. 
We will use $\Kh$ as our Gaussian process covariance function from now on, replacing $K$.
This results in the following predictive distribution:
\[
p(\Y | \X) &= \int p(\Y | \F) p(\F | \bo, \X) p(\bo) \td \bo \td \F.
\]

We can integrate this analytically for $\F$ and obtain
\[
p(\Y | \X) 
&= \int \N(\Y; \bz, \Kh(\X,\X) + \tau^{-1} \I) p(\bo) \td \bo 
\]
but this involves the inversion of $\Kh(\X,\X) + \tau^{-1} \I$, which does not allow us to integrate over $\bo$ (even variationally!). Instead, we introduce an auxiliary random variable. 

Denoting the $1 \times K$ row vector 
\[
\phi(\x, \bo) = \bigg[ \sqrt{\frac{2\sigma^2}{K}} \cos \big(2 \pi \w_k^T(\x - \z_k) + b_k \big) \bigg]_{k=1}^K
\]
and the $N \times K$ feature matrix $\Phi = [\phi(\x_n, \bo)]_{n=1}^N$, we have $\Kh(\X,\X) = \Phi\Phi^T$.
We rewrite $p(\Y | \X)$ as 
\[
&p(\Y | \X) = \int \N(\Y; \bz, \Phi\Phi^T + \tau^{-1} \I) p(\bo) \td \bo.
\]
Following identity \citep[page 93, equations 2.113 $-$ 2.115]{Bishop2006Pattern} we introduce a $K \times 1$ auxiliary random variable $\ba_d \sim \N(0, \I_K)$ to the distribution inside the integral above,
\[
&\N(\y_d; \bz, \Phi\Phi^T + \tau^{-1} \I) \notag\\
&\qquad = \int \N(\y_d; \Phi\ba_d, \tau^{-1} \I) \N(\ba_d; 0, \I_K) \td \ba_d,
\]
where $\y_d$ is the $d$'th column of the $N \times D$ matrix $\Y$. 

Writing $\A = [\ba_d]_{d=1}^D$, the above is equivalent to\footnote{This is equivalent to the weighted basis function interpretation of the Gaussian process \citep{Rasmussen2005Gaussian}.} 
\fl[\label{eq:Y_given_A_X_o}
p(\Y | \X) 
&= \int p(\Y | \A, \X, \bo) p(\A) p(\bo) \td \A \td \bo.
\]
We refer to  $\A \in \R^{K \times D}$ as the Fourier coefficients. 

Regarding $\bo$ as parameters and optimising these values (integrating over $\A$) results in the \textit{sparse spectrum} approximation \citep{lazaro2010sparse}. 
Regarding $\A$ as parameters and optimising these values (leaving $\bo$ constant) results in a method known as ``random projections'' \citep{rahimi2007random}.
Related work to random projections variationally integrates over the hyper-parameters while leaving $\bo$ constant \citep{Tan2013Variational}.

We can extend the above to sums of covariance functions as well. Following proposition \ref{prop:2} in the appendix, given a sum of covariance functions with $L$ components (with each corresponding to $\Phi_i$ an $N \times K$ matrix) we have $\Phi = [\Phi_i]_{i=1}^L$ an $N \times LK$ matrix. 

As an example covariance function of this form consider the spectral mixture (SM) covariance function \citep{lindgren2012stationary,wilson2013gaussian}. 
This covariance function has been used in the audio processing community since the '70s and was recently introduced to the machine learning community. It generalises many known covariance functions, such as the periodic covariance function, the automatic relevance determination (ARD) squared exponential (SE) covariance function, products of these and weighted sums of these products. 

We will continue the development of our method using this covariance function. Note however that our method is general and can be extended for other covariance functions as well.
The spectral mixture covariance function with $L$ components is given by
\[
K(\x, \y) = \sum_{i=1}^L \sigma_i^2 
&\exp \bigg(
-\frac{1}{2} \sum_{q=1}^Q \frac{(x_q - y_q)^2}{l_{iq}^2}
\bigg) \notag\\
& \cdot \prod_{q=1}^Q
\cos \bigg( \frac{2 \pi (x_q - y_q)}{p_{iq}} \bigg)
\]
with weights $\sigma_i^2$, length-scales $l_{iq}$ and periods $p_{iq}^{-1}$. We write $\bp_i = [p_{iq}^{-1}]_{q=1}^Q$ and $\bL_i = \diag([2 \pi l_{qi} ]_{q=1}^Q)$. This covariance function reduces to a sum of squared exponential (SE) covariance functions for $p_{iq} = \infty$ for all $i$ and $q$.

For $p(\w)$ composed of a single SM component, we follow proposition \ref{prop:3} in the appendix and perform a change of variables, resulting in $p(\w)$ a standard normal distribution with the parameters of $p(\w)$ now expressed in $\Phi$ instead.
For $p(\w)$ composed of several components, for each component $i$ we get $\Phi_i$ is an $N \times K$ matrix with elements
\[
\sqrt{\frac{2\sigma_i^2}{K}} \cos \big(2 \pi (\bL_i^{-1} \w_k + \bp_i)^T (\x - \z_k) + b_k \big),
\] 
where for simplicity, we index $\w_k$ and $b_k$ with $k=1,...,LK$ as a function of $i$.

\section{Variational Inference}
The predictive distribution for an input point $\x^*$ is given by 
\fl[\label{eq:pred_dist}
&p(\y^* | \x^*, \X, \Y) = \int p(\y^* | \x^*, \A, \bo) p(\A, \bo | \X, \Y)\td \A \td \bo,
\]
with $\y^* \in \R^{1 \times D}$.
The distribution $p(\A, \bo | \X, \Y)$ cannot be evaluated analytically. Instead we define an approximating \textit{variational} distribution $q(\A, \bo)$, whose structure is easy to evaluate.

We would like our approximating distribution to be as close as possible to the posterior distribution obtained from the full GP. We thus minimise the Kullback--Leibler divergence 
\[
\KL(q(\A, \bo) ~|~ p(\A, \bo | \X, \Y)),
\]
resulting in the approximate predictive distribution 
\fl[\label{eq:approx_pred_dist}
&q(\y^* | \x^*) = \int p(\y^* | \x^*, \A, \bo) q(\A, \bo) \td \A \td \bo.
\]

Minimising the Kullback--Leibler divergence is equivalent to maximising the log evidence lower bound
\fl[
&\cL := \int q(\A, \bo) \log p(\Y | \A, \X, \bo) \td \A \td \bo \notag\\
&\qquad \qquad \qquad \qquad 
- \KL(q(\A, \bo) || p(\A) p(\bo)) \label{eq:lower_bound}
\]
with respect to the variational parameters defining $q(\A,\bo)$. 

We define a factorised variational distribution $q(\A,\bo) = q(\A) q(\bo)$. We define $q(\bo)$ with $\bo = (\w_k, b_k)_{k=1}^K$ to be a joint Gaussian distribution and a uniform distribution,
\[
\w_k &\sim \N(\mu_k, \Sigma_k), &&k = 1, ..., LK \\
b_k &\sim \Unif(\alpha_k, \beta_k), &&k = 1, ..., LK
\]
with $\Sigma_k$ diagonal, $0 \leq \alpha_k \leq \beta_k \leq 2 \pi$, and define $q(\A) = \prod_{d=1}^D q(\ba_d)$ (with $\ba_d \in \R^{LK \times 1}$) by
\[
\ba_d \sim \N(\m_d, \s_d), &&d = 1, ..., D
\]
with $\s_d$ diagonal. We evaluate the log evidence lower bound and optimise over $\{ \mu_k, \Sigma_k, \alpha_k, \beta_k \}_{k=1}^{LK}$, $\{ \m_d, \s_d \}_{d=1}^D$, and $\{ \sigma_i, \bL_i, \bp_i \}_{i=1}^L$ to maximise Eq.\ \ref{eq:lower_bound}.

\subsection{Evaluating the Log Evidence Lower Bound}

Given $\A$ and $\bo$, we evaluate the probability of the $d$'th element, $\y_d$, as
\[
&\log p(\y_d | \A, \X, \bo) = 
\notag\\
&\qquad
-\frac{N}{2} \log(2 \pi \tau^{-1}) - \frac{\tau}{2} (\y_d - \Phi\ba_d)^T(\y_d - \Phi\ba_d).
\]
Note that $\y_d$ is an $N \times 1$ vector, $\Phi$ is an $N \times LK$ matrix, and $\ba_d$ is an $LK \times 1$ vector.

We need to evaluate the expectations of $\y_d^T\Phi\ba_d$ and $\ba_d^T\Phi^T\Phi\ba_d$ (both scalar values) under $q(\A)q(\bo)$:
\fl[\label{eq:exp_a_Phi}
&E_{q(\A)q(\bo)}\big( \y_d^T\Phi\ba_d \big) 
 =
\y_d^T E_{q(\bo)}\big( \Phi \big) E_{q(\A)}\big( \ba_d \big), 
\]
and
\fl[\label{eq:exp_a_Phi_Phi_a}
E_{q(\A)q(\bo)}\big( 
\ba_d^T\Phi^T\Phi\ba_d
 \big) 
 &= 
\tr \bigg(
	E_{q(\bo)}\big( 
		\Phi^T\Phi
	 \big) 
	E_{q(\A)}\big( 
		 \ba_d \ba_d^T
	 \big) 
\bigg).
\]

The values $E_{q(\A)}( \ba_d )$ and $E_{q(\A)}( \ba_d \ba_d^T )$ are evaluated as 
\[
E_{q(\A)}( \ba_d ) &= \m_d, \\
E_{q(\A)}( \ba_d \ba_d^T ) &= \s_d + \m_d \m_d^T.
\]

Next we evaluate $E_{q(\bo)}\big(\Phi)$. Remember that $\Phi$ depends on $\bo$ and that $q(\bo) = q((\w_k, b_k)_{k=1}^{LK})$. We write as shorthand $\bx_{nk} := 2\pi \bL_i^{-1} (\x_n - \z_k)$ and $\bb_{nk} = b_k + 2 \pi \bp_i^T(\x_n-\z_k)$ with component $i$ appropriate to $k$.
Following identity \ref{identity:2} proved in the appendix, we have that 
the expectation of a single element in the vector with respect to $q(\w_k)$ is
\[
&E_{q(\w_k)}\big(\cos \big(\w_k^T \bx_{nk} + \bb_{nk} \big)\big) \notag\\
& \quad = e^{-\frac{1}{2} \bx_{nk}^T \Sigma_k \bx_{nk}} \cos\big(\mu_k^T \bx_{nk} + \bb_{nk} \big).
\]
where $\mu_k$ is the mean of $q(\w_k)$ and $\Sigma_k$ is its covariance.
We get 
\fl[\label{eq:exp_Phi}
\bigg( E_{q(\bo)}\big( \Phi \big) \bigg)_{n,k} &=
\sqrt{\frac{2\sigma_i^2}{K}} e^{-\frac{1}{2} \bx_{nk}^T \Sigma_k \bx_{nk}}
\notag \\ &\qquad \cdot
E_{q(b_k)} \big( \cos(\mu_k^T \bx_{nk} + \bb_{nk}) \big)
\]
with the integration with respect to $q(b_k)$ trivial.

Next we evaluate $E_{q(\bo)}\big( \Phi^T\Phi \big)$, an $LK \times LK$ matrix:
\fl[\label{eq:exp_PhiT_Phi}
&\bigg( E_{q(\bo)}\big( \Phi^T\Phi \big) \bigg)_{i, j} \notag \\ &\quad
= \sum_{n=1}^N \frac{2 \sigma_i^2}{LK} 
E_{q(\w_i, b_i, \w_j, b_j)}\big( \cos(\w_i^T \bx_{ni} + \bb_{ni}) \notag\\
&\qquad\qquad\qquad\qquad\qquad\qquad \cdot \cos(\w_j^T \bx_{nj} + \bb_{nj})\big)
\]
for $i, j \leq LK$. 

For $i \neq j$, from independence we can break the expectation of each term into 
\[
&E_{q(\w_i, b_i, \w_j, b_j)}\big(\cos(\w_i^T \bx_{ni} + \bb_{ni}) \cos(\w_j^T \bx_{nj} + \bb_{nj})\big) 
\notag \\ &\quad =
E_{q(\w_i, b_i)}\big(\cos(\w_i^T \bx_{ni} + \bb_{ni})\big)
\notag\\ &\quad\quad \cdot
E_{q(\w_j, b_j)}\big(\cos(\w_j^T \bx_{nj} + \bb_{nj})\big),
\]
and for $i = j$,
\[
&E_{q(\w_i, b_i)}\big(\cos(\w_i^T \bx_{ni} + \bb_{ni})^2\big) = \frac{1}{2} + 
\notag \\ &\qquad\qquad
\frac{1}{2} e^{-2 \bx_{ni}^T \Sigma_i \bx_{ni}} E_{q(b_i)} \big( \cos(2\mu_i^T \bx_{ni} + 2\bb_{ni}) \big) 
\]
following identity \ref{identity:3}. 

In conclusion, we obtained our optimisation objective:
\fl[\label{eq:lower_bound_factorised}
&\cL = 
\sum_{d=1}^D \bigg( -\frac{N}{2} \log(2 \pi \tau^{-1}) -\frac{\tau}{2} \y_d^T\y_d
\notag \\ &\qquad \qquad \quad 
+\tau \y_d^T E_{q(\bo)}\big( \Phi \big) \m_d \notag\\
&\qquad \qquad \quad -\frac{\tau}{2} 
	\tr \big(
		E_{q(\bo)}( 
			\Phi^T\Phi
		) 
		( 
			 \s_d + \m_d \m_d^T
		) 
	\big)
 \bigg) 
 \notag \\ &\qquad 
- \KL(q(\A) || p(\A)) - \KL(q(\bo) || p(\bo)). 
\]

The KL divergence terms can be evaluated analytically for the Gaussian and uniform distributions. 

\subsection{Optimal variational distribution over $\A$}

In the above we optimise over the variational parameters for $\A$, namely $\m_d$ and $\s_d$ for $d \leq D$. This allows us to attain a reduction in time complexity compared to previous approaches and use stochastic inference, as will be explained below. This comes with a cost, as the dependence between $\bo$ and $\A$ can render the optimisation hard.

We can find the optimal variational distribution $q(\A)$ analytically, which allows us to optimise $\bo$ and the hyper-parameters alone. In proposition \ref{prop:1} in the appendix we show that the optimal variational distribution is given by
\[
q(\ba_d) = \N(\bSigma E_{q(\bo)}(\Phi^T) \y_d, ~\tau^{-1} \bSigma)
\]
with $\bSigma = (E_{q(\bo)}(\Phi^T \Phi) + \tau^{-1} I)^{-1}$.

The lower bound to optimise then reduces to
\fl[\label{eq:lower_bound_opt_A}
&\cL = 
\sum_{d=1}^D \bigg( -\frac{N}{2} \log(2 \pi \tau^{-1}) -\frac{\tau}{2} \y_d^T\y_d
+ \frac{1}{2} \log(|\tau^{-1} \bSigma|) 
\notag\\
&\qquad \qquad \quad 
+ \frac{1}{2} \tau \y_d^T E_{q(\bo)}(\Phi) \bSigma E_{q(\bo)}(\Phi^T) \y_d
\bigg)
 \notag \\ &\qquad 
- \KL(q(\bo) || p(\bo)). 
\]


\section{Distributed Inference and Stochastic Inference}

Evaluating $\cL$ in equation \ref{eq:lower_bound_factorised} requires $\cO(N K^2)$ time complexity (for fixed $Q,D$, diagonal $\s_d$, and covariance function with one component $L=1$). This stems from the term $E_{q(\bo)}\big(\Phi^T\Phi\big)$ -- a $K \times K$ matrix, where each element is composed of a sum over $N$.

Following the ideas of \citep{Gal2014DistributedB} we show that the approximation can be implemented in a distributed framework.
Write 
\[
\cL_{nd} &= 
-\frac{1}{2} \log(2 \pi \tau^{-1}) -\frac{\tau}{2} y_{nd} y_{nd}
+\tau y_{nd} E_{q(\bo)}\big( \phi_n \big) \m_d  \notag\\
&\qquad \quad -\frac{\tau}{2} 
	\tr \big(
		E_{q(\bo)}( 
			\phi_n^T\phi_n
		) 
		( 
			 \s_d + \m_d \m_d^T
		) 
	\big)
\]
with $\phi_n = \phi(\x_n, \bo)$.
We can break the optimisation objective in equation \ref{eq:lower_bound_factorised} into a sum over $N$, 
\fl[\label{eq:lower_bound_dist}
\cL & = 
\sum_{d=1}^D \sum_{n=1}^N \cL_{nd}
- \KL(q(\A) || p(\A)) - \KL(q(\bo) || p(\bo)). 
\]
These terms can be computed concurrently on different nodes in a distributed framework, requiring $\cO(K^2)$ time complexity in each iteration. We can further break the computation of $\cL_{nd}$ into a sum over $K$ as well, thus reducing the time complexity to $\cO(K)$ with $K$ inducing points. This is in comparison to distributed inference with sparse pseudo-input GPs which takes $\cO(K^3)$ time complexity with $K$ inducing points, resulting from the covariance matrix inversion. This is a major advantage, as empirical results suggest that in many real-world applications the number of inducing points should scale with the data.

We can exploit the above representation and perform stochastic variational inference (SVI) by approximating the objective with a subset of the data, resulting in noisy gradients \citep{Hoffman2013Stochastic}.
Here we use as our objective
\fl[\label{eq:lower_bound_SVI}
\cL &\approx
\frac{N}{|S|}
\sum_{d=1}^D \sum_{n \in S} \cL_{nd}
 \notag \\ &\qquad 
- \KL(q(\A) || p(\A)) - \KL(q(\bo) || p(\bo)). 
\]
with a mini-batch $S$ of randomly selected points. This is an unbiased estimator to the lower bound. The time complexity of each iteration is $\cO(SK^2)$ with $S<<N$ the size of the random subset, compared to $\cO(SK^2+K^3)$ of GP SVI using sparse pseudo-input approximation \citep{hensman2013Gaussian}.

\section{Predictive Distribution}

The approximate predictive distribution for a point $\x^*$ is given by equation \ref{eq:approx_pred_dist}.
Denoting $\M = [\m_d]_{d=1}^D$, we have 
\fl[\label{eq:pred_mean}
E_{q(\y^* | \x^*)}(\y^*) = E_{q(\bo)} \big( \phi_* \big) \M
\]
following proposition \ref{prop:4} in the appendix.

The variance of the predictive distribution is given by
\fl[\label{eq:pred_uncertainty}
&\Var_{q(\y^* | \x^*)}(\y^*) 
=
\tau^{-1}\I_D + \Psi  \\ 
&\qquad + \M^T \big( E_{q(\bo)}\big(\phi_*^T \phi_*\big) - E_{q(\bo)} \big( \phi_* \big)^T E_{q(\bo)} \big( \phi_* \big)\big) \M\notag
\]
with $\Psi_{i,j} = \tr \big( E_{q(\bo)}\big(\phi_*^T \phi_*\big) \cdot \s_i \big) \cdot \mathds{1}[i=j]$, following proposition \ref{prop:5} in the appendix ($\mathds{1}$ is the indicator function).

When the optimal variational distribution over $\A$ is used, we have $\M  = \bSigma E_{q(\bo)}(\Phi^T) \Y$ and $\s_i = \tau^{-1} \bSigma$ for all $i$.

\section{Properties of the Approximate Model}

We have presented a variational sparse spectrum approximation to the Gaussian process (VSSGP in short). We gave 3 approximate models with different lower bounds: an approximate model with an optimal variational distribution over $\A$ (equation \ref{eq:lower_bound_opt_A}, referred to as VSSGP), an approximate model with a \textit{factorised} lower bound over the data points (equations \ref{eq:lower_bound_factorised}, \ref{eq:lower_bound_dist}, referred to as factorised VSSGP -- fVSSGP), and an approximation to the lower bound of the factorised VSSGP for use in stochastic optimisation over subsets of the data (equation \ref{eq:lower_bound_SVI}, referred to as stochastic factorised VSSGP -- sfVSSGP).

The VSSGP model generalises on some of the GP approximations brought in the introduction.
Fixing $\Sigma_k$ at zero in our approximate model (as well as $\alpha_k$ and $\beta_k$ at 0 and $2 \pi$) and optimising only $\mu_k$ results in the \textit{sparse spectrum approximation}. 
Randomising $\mu_k$, we obtain the \textit{random projections} approximation \citep{rahimi2007random}. 
Indeed, for $\Sigma_k = \bz$ and fixed phases we have that $E_{q(\bo)}\big( \Phi^T\Phi \big) = E_{q(\bo)}\big( \Phi^T \big) E_{q(\bo)}\big( \Phi \big)$ and $E_{q(\bo)}\big( \Phi \big) = \Phi$, and equation \ref{eq:lower_bound_opt_A} recovers equation 8 in \cite{lazaro2010sparse}.

The points $\z_k$ act as \textit{inducing inputs} with $\w_k$ and $b_k$ acting as \textit{inducing frequencies and phases} at these inputs. This is similar to the \textit{sparse pseudo-input} approximation, but instead of having inducing values in the output space, we have the inducing values in the frequency domain. These are necessary to the approximation. Without these points (or equivalently, setting these to $\bz$), the features would decay quickly for data points far from the origin (the fixed point $\bz$).
~\\

The distribution over the frequencies is optimised to fit the data well. The prior is used to regulate the fit and avoid over-fitting to the data. This approximation can be used to learn covariance functions by fitting them to the data. This is similar to the ideas brought in \citep{DuvLloGroetal13} where the structure of a covariance function is sought by looking at possible compositions of these. This can give additional insight into the data. In \citep{DuvLloGroetal13} the structure of the covariance composition is used to explain the data. In the approximation presented here the spectrum of the covariance function can be used to explain the data.

It is interesting to note that although the approximated covariance function $K(\x, \y)$ has to be stationary (i.e.\ it is represented as $K(\x,\y) = K(\x-\y)$), the approximate posterior is not. This is in contrast to the SSGP that results in a stationary approximation.
Furthermore, unlike the SSGP, our approximation is not periodic. This is one of the theoretical limitations of the sparse spectrum approximation. The limitation arises from the fact that the covariance is represented as a weighted sum of cosines in the SSGP. In the our approximation this is avoided by decaying the cosines to zero.
This and other properties of the approximation are discussed further in discussion \ref{dis:1} in the appendix.

\section{Experiments}

We next study the properties of the VSSGP and compare it to alternative approximations, showing its advantages.
We compare the VSSGP to the full Gaussian process 
(denoted Full GP)
, the sparse spectrum GP approximation 
(denoted SSGP)
, a sparse pseudo-input GP approximation 
(denoted SPGP)
, and the random projections approximation 
(denoted RP)
. We compare the VSSGP to the fVSSGP and sfVSSGP that offer improved time complexity. We further compare sfVSSGP to the existing sparse pseudo-input GP approach used with SVI \citep[denoted sSPGP, ][]{hensman2013Gaussian}.
We inspect the model's time accuracy trade-off and show that it avoids over-fitting as the number of parameters increases.

\subsection{VSSGP Properties}

We evaluate the predictive mean and uncertainty of the VSSGP on the atmospheric CO$_2$ concentrations dataset derived from in situ air samples collected at Mauna Loa Observatory, Hawaii \citep{Keeling2004}. We fit the approximate model using a spectral mixture covariance function with two components initialised with periods $[5, \infty]$ and corresponding initial length-scales $[0.1, 1000]$ (resulting in a sum of SE $\times$ periodic and SE covariances). We randomised the phases following the Monte Carlo integration (instead of optimising a variational distribution on these\footnote{This seems to work better in practice.}) and initialise the frequencies at random. 
We use 10 inducing inputs for each component ($K=10$), set the observation noise precision to $\tau = 10$, and covariance noise to $\sigma^2 = 1$. LBFGS \citep{zhu1997algorithm} was used to optimise the objective given in equation \ref{eq:lower_bound_opt_A}, and was run for 500 iterations.

\begin{figure}[t!]
\vspace{-3mm}
\hspace{-5mm}
\includegraphics[width=1.08\linewidth]{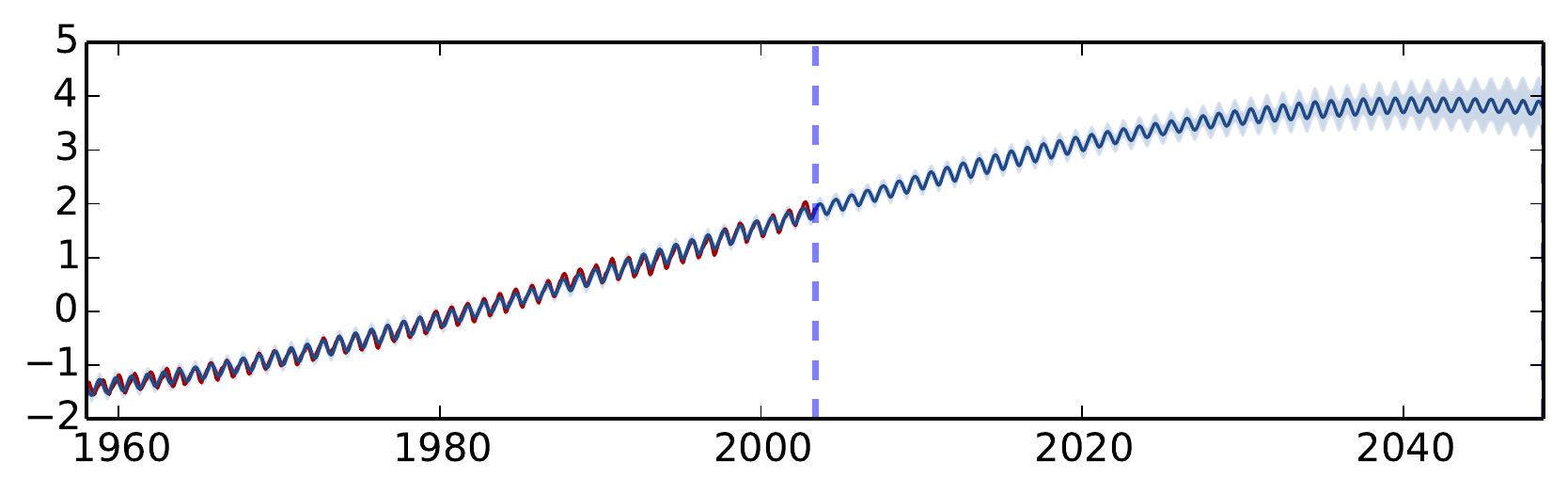}
\vspace{-8mm}
\caption{Predictive mean and uncertainty on the Mauna Loa CO$_2$ concentrations dataset.
In red is the observed function; in blue is the predictive mean plus/minus two standard deviations. In this example the approximating distribution is used with a spectral mixture covariance with two components ($L=2$, $K=10$). }\label{fig:exp1}
\vspace{-6mm}
\end{figure}

Figure \ref{fig:exp1} shows the predictive mean with the predictive uncertainty increasing far from the data. This is a property shared with the SE GP. The covariance hyper-parameters optimise to periods of $[9.8, \infty]$, length-scales $[0.09, 54]$, and covariance noise $[0.0043, 5.7]$, correspondingly. The frequency with the smallest standard-deviation (highest confidence) for the first component is $1$ (corresponding to a period of $1$ year, capturing the short term behaviour). For the second component these are $0.0053, 0.00065$ (corresponding to periods of $185$ and $1536$ years capturing the long term behaviour).

\subsection{Comparison to Existing GP Approximations}

We compare various GP approximations on the solar irradiance dataset \citep{Lean2004}. We scaled the dataset dividing by the data standard deviation, and removed 5 segments of length 20.
We followed the experiment set-up of the previous section and used the same initial parameters for all approximate models. Instead of the SM covariance function we use a single SE setting its length-scale $l=1$, and used 50 inducing inputs. LBFGS was used for 1000 iterations. The RP model was run twice with two different settings: once following the same set-up of the other models, optimising over the model hyper-parameters (RP$_1$), and once keeping all hyper-parameters fixed and setting the observation noise precision to $\tau = 100$ with $K=500$ inducing inputs\footnote{This follows the usual use of the model in the randomised methods community. We experimented with various values of $\tau$ and decided to use 100.} (RP$_2$).

\begin{figure}[t!]
\vspace{-3mm}
\center
\subfloat[Sparse pseudo-input GP]{
\includegraphics[width=1\linewidth, trim = 3mm 3mm 2mm 2mm, clip]{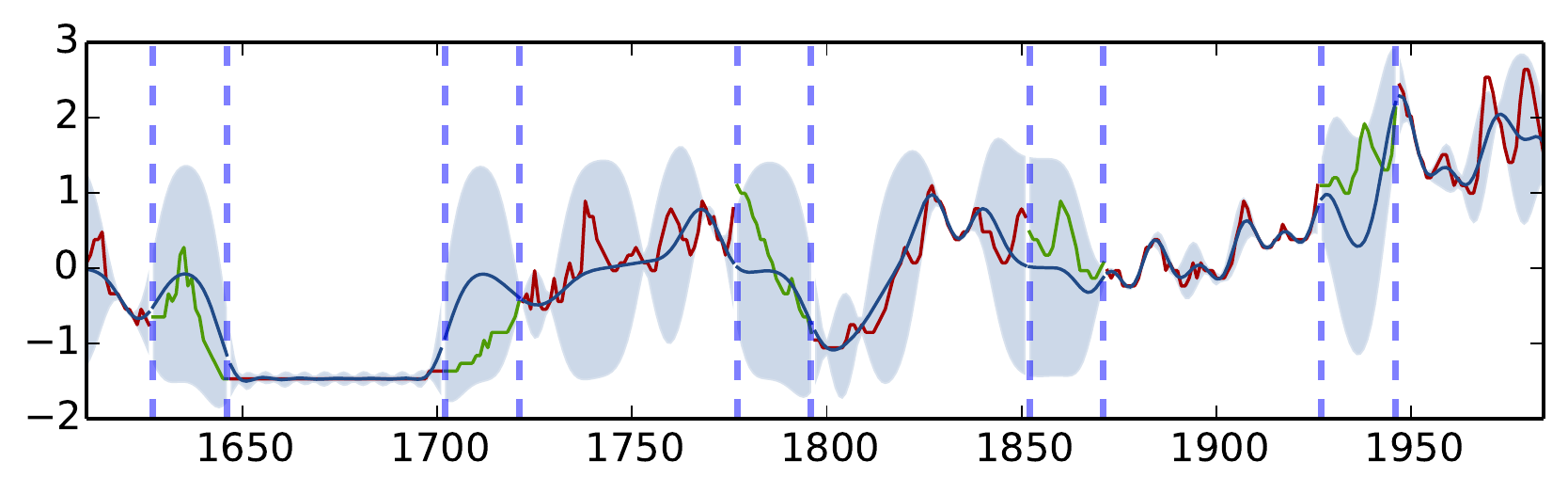}
}
\\
\vspace{-4mm}
\subfloat[Sparse Spectrum GP]{
\includegraphics[width=1\linewidth, trim = 3mm 3mm 2mm 2mm, clip]{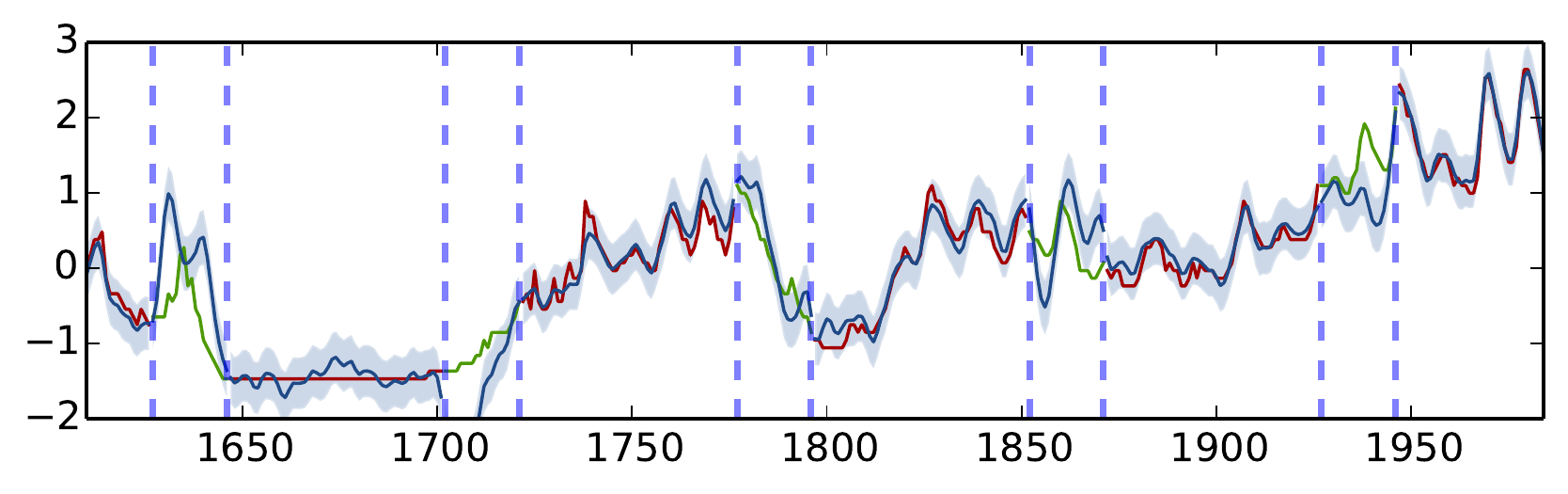}
}
\\
\vspace{-4mm}
\subfloat[Random Projections (RP$_2$, $K = 500$)]{
\includegraphics[width=1\linewidth, trim = 3mm 3mm 2mm 2mm, clip]{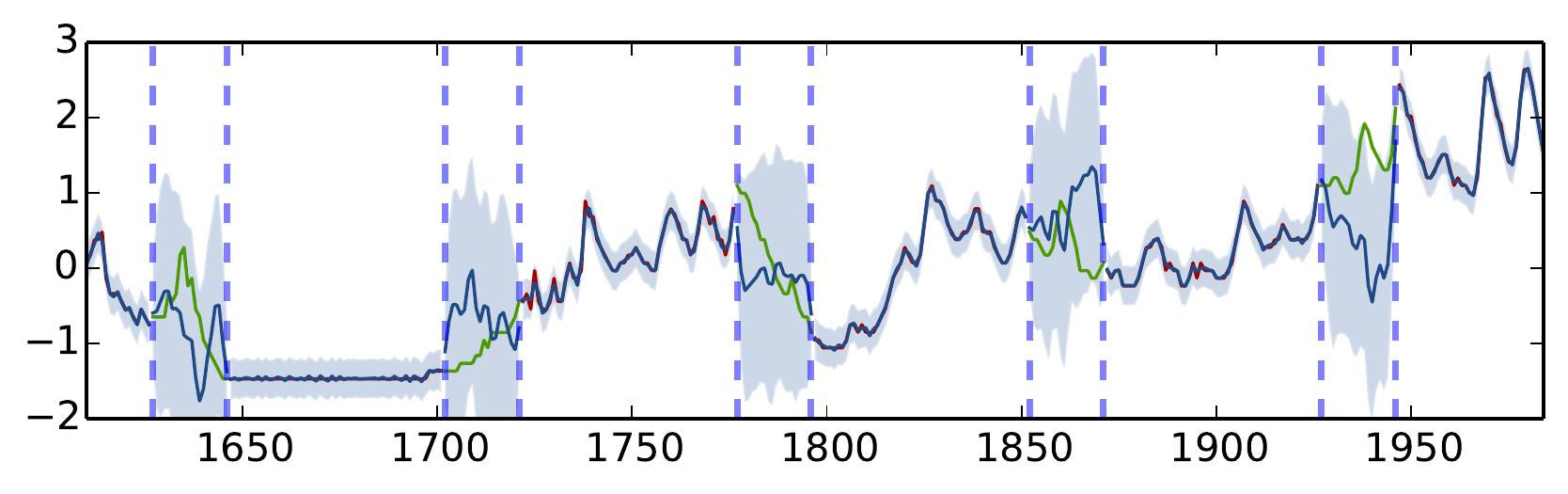}
}
\\
\vspace{-4mm}
\subfloat[Full GP]{
\includegraphics[width=1\linewidth, trim = 3mm 3mm 2mm 2mm, clip]{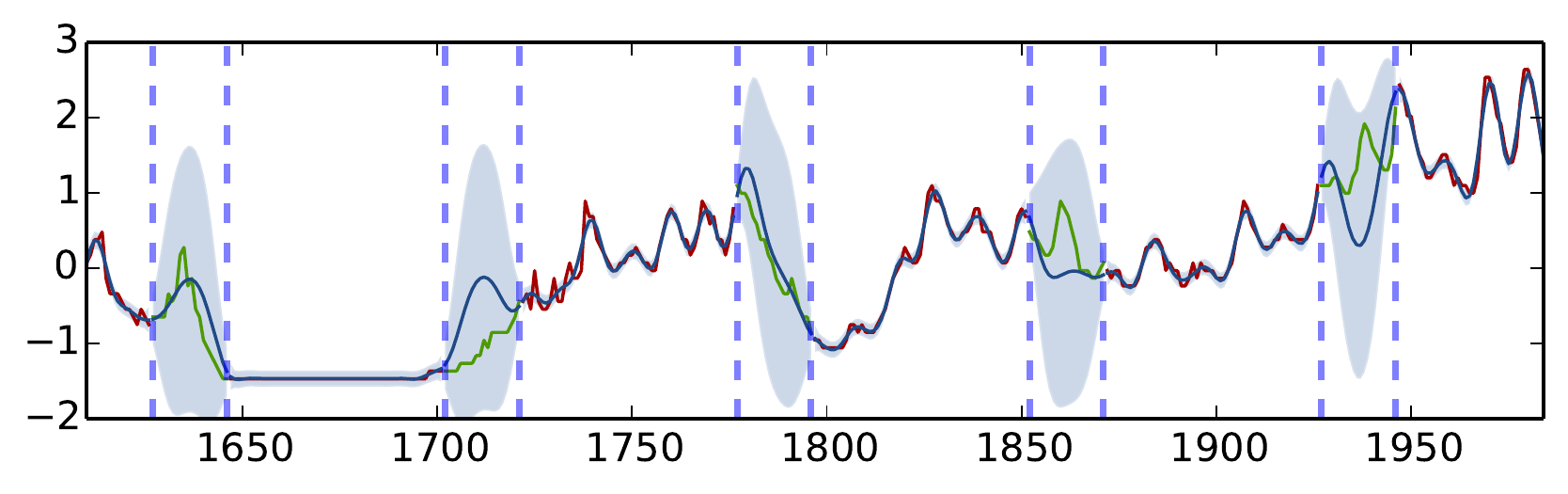}
}
\\
\vspace{-4mm}
\subfloat[\textbf{Variational Sparse Spectrum GP}]{
\includegraphics[width=1\linewidth, trim = 3mm 3mm 2mm 2mm, clip]{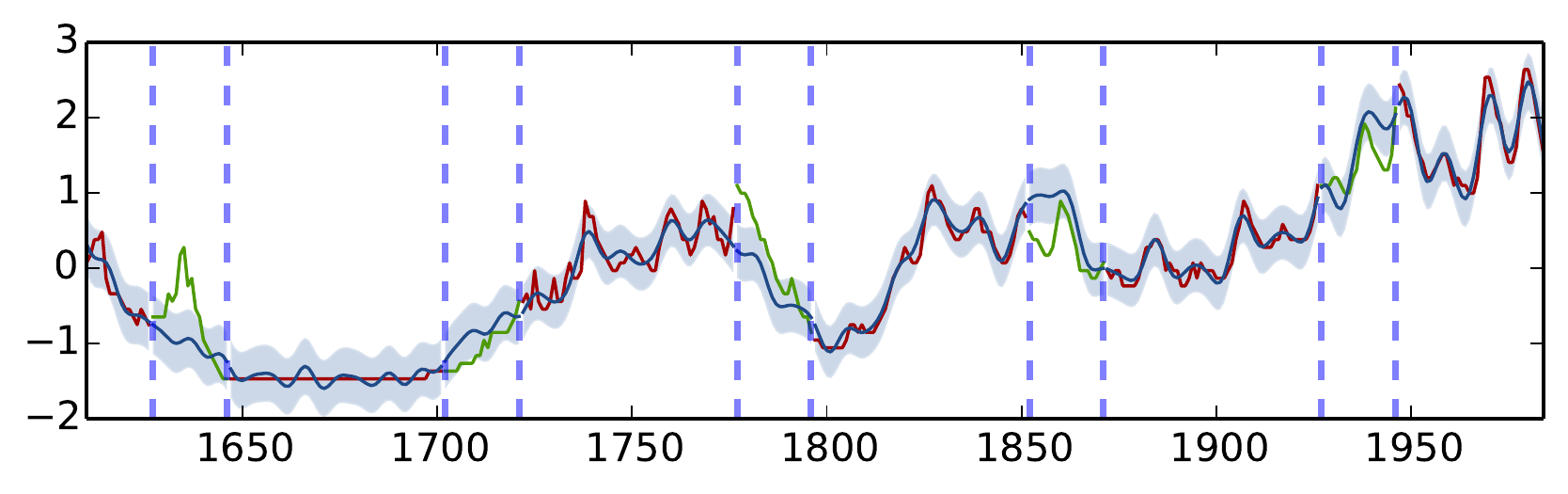}
}
\vspace{-2mm}
\caption{Predictive mean and uncertainty on the reconstructed solar irradiance dataset with missing segments, for the GP and various GP approximations. In red is the observed function and in green are the missing segments. In blue is the predictive mean plus/minus two standard deviations of the various approximations. All tests were done with the SE covariance function, and all sparse approximations use $K=50$ inducing inputs (apart from RP$_2$ with $K=500$).}\label{fig:exp2a}
\vspace{-5mm}
\end{figure}

Figure \ref{fig:exp2a} shows qualitatively the predictive mean and uncertainty of the various approaches. 
SSGP and RP seem to over-fit the function using high frequencies with high confidence. 
SPGP seems to under-fit the function, but has accurate predictive mean and uncertainty at points where many inducing inputs lie (such as the flat region).
VSSGP's predictive mean resembles that of the full GP, but with increased uncertainty throughout the space. Further, its uncertainty on the missing segments is smaller than that of the full GP (some frequencies have low uncertainty, thus used near the data).
The full GP learnt length-scale is $4$. VSSGP learnt a length-scale of $3$, and SPGP learnt a length-scale of $5$. SSGP and RP$_1$ learnt length-scales of $0.97,1.66$, i.e.\ the hyper-parameter optimisation found a local minimum.

\begin{table}[t!]
\vspace{-3mm}
\center
\begin{tabular}{|c|c|c|c|c|c|c|}
\hline 
Solar & SPGP & SSGP & RP$_1$ & RP$_2$ & GP & \textbf{VSSGP} \\ 
\hline 
Train & 0.23 & 0.15 & 0.32 & 0.04 & 0.08 & \textbf{0.13} \\ 
\hline 
Test & 0.61 & 0.63 & 0.65 & 0.76 & 0.50 & \textbf{0.41} \\ 
\hline 
\end{tabular} 
\caption{Imputation RMSE on both train and test sets, for the reconstructed solar irradiance dataset. All tests were done with the SE covariance function, and all sparse approximations use 50 inducing inputs (apart from RP$_2$ that uses $K=500$).}
\vspace{-5mm}
\label{table:exp2b}
\end{table}

Table \ref{table:exp2b} gives a quantitative comparison of the different approximations for the task of imputation. RMSE (root mean square error) of the approximate predictive mean on the missing segments was computed (test error), as well as the RMSE on the observed function (training error). Note that the full GP seems to get worse results than VSSGP. This might be because of the (slightly) larger learnt length-scale.

\subsection{From SSGP to Variational SSGP}

We use of variational inference in the VSSGP to avoid over-fitting to the data, a behaviour that is often observed with SSGP. 
To test this we perform a direct comparison of the proposed approximate model to SSGP on the task of audio signal imputation. 
For this experiment we used a short speech signal with 1000 samples taken from the TIMIT dataset \citep{garofolo1993timit}. We removed 5 segments of length 40 from the signal, and evaluated the imputation error (RMSE) of the predictive mean with $K=100$ inducing points. We used the same experiment set-up as before with a sum of 2 SE covariance functions with length-scales $l=[2,10]$ and observation noise precision $\tau = 1000$ matching the signal magnitude. LBFGS was run for 1000 iterations. The experiment was repeated 5 times and the results averaged.

\begin{table*}[t!]
\center
\begin{tabular}{|c|c|c|c|}
\hline 
Audio 1K & VSSGP & fVSSGP & sfVSSGP \\ 
\hline 
Train & $\mathbf{0.0062 \pm 0.00048}$ $_{(0.063 \pm 0.0068)}$ & $\mathbf{0.0054 \pm 0.00083}$ $_{(0.055 \pm 0.0088)}$ & $\mathbf{0.005 \pm 0.003}$ $_{(0.052 \pm 0.031)}$ \\ 
\hline 
Test & $\mathbf{0.034 \pm 0.0043}$ $_{(0.17 \pm 0.022)}$ & $\mathbf{0.038 \pm 0.0049}$ $_{(0.22 \pm 0.028)}$ & $\mathbf{0.04 \pm 0.0066}$ $_{(0.24 \pm 0.0089)}$ \\ 
\hline 
\end{tabular} 
\caption{Imputation RMSE (and in smaller font STFT RMSE) on train and test sets, for a speech signal segment of length 1K ($K=100$).}
\vspace{-4mm}
\label{table:exp3}
\end{table*}

\begin{table}[h]
\vspace{-1mm}
\center
\begin{tabular}{|c|c|c|}
\hline 
Audio 1K & SSGP & \textbf{VSSGP} \\ 
\hline 
Train & $0.0091 \pm 0.0042$ & $\mathbf{0.0062 \pm 0.00048}$ \\ 
\hline 
Test & $0.088 \pm 0.033$ & $\mathbf{0.034 \pm 0.0043}$ \\ 
\hline 
\end{tabular} 
\caption{Imputation RMSE on both train and test sets, for a speech signal segment of length 1K ($K=100$).}
\label{table:exp2c}
\vspace{-1mm}
\end{table}

Table \ref{table:exp2c} shows the RMSE of the training set and test set for the audio data. SSGP seems to achieve a small training error but cannot generalise well to unseen audio segments. VSSGP attains a slightly lower training error, and is able to impute unseen audio segments with better accuracy. 

It is interesting to note that using the RMSE of the short-time Fourier transform of the original signal and the predicted mean (STFT, the common metric for audio imputation, with 25ms frame size and a hop size of 12ms), the SSGP model attains a training error of $0.094 \pm 0.05$ and a test error of $0.55 \pm 0.41$. The VSSGP attains a training error of $0.067 \pm 0.0067$ with a test error of $0.17 \pm 0.022$. For comparison, baseline performance of predicting 0 attains an error of $0.44$ on the training set and an error of $0.38$ on the test set.

\subsection{VSSGP, factorised VSSGP, and stochastic factorised VSSGP}

VSSGP, fVSSGP, and sfVSSGP all rely on different lower bounds to the same approximate model. Whereas VSSGP solves for the variational distribution over the Fourier coefficients analytically, fVSSGP optimises over these quantities. This reduces the time complexity, but with the price of potentially worsened performance. sfVSSGP further employs an approximation to the lower bound using random subsets of the data -- following the idea that not all data points have to be observed for a good fit to be found. This assumption has the potential to hinder performance even further. We next assess these trade-offs.

We repeated the experimental set-up of the previous section (and use the same RMSE for VSSGP). We optimise both fVSSGP and sfVSSGP for 5000 iterations instead of the 1000 of VSSGP. This is because the improved time complexity allows us to perform more function evaluations within the same time-frame. We optimise the fVSSGP lower bound with LBFGS, and the sfVSSGP lower bound with RMSPROP \citep{Tieleman2012COURSERA}. RMSPROP performs stochastic optimisation with no need for learning-rate tuning -- the learning rate changes adaptively based on the directions of the last two gradients.

Table \ref{table:exp3} shows the RMSE for the train and test sets. Both fVSSGP and sfVSSGP effectively achieve the same test set accuracy (taking the standard deviation into account). We also see a slight decrease in train set RMSE. 

\subsection{Stochastic Variational Inference}

We compared sfVSSGP to the SPGP approximation with stochastic variational inference \citep[sSPGP, ][]{hensman2013Gaussian}. We used the same audio experiment as above, but with a signal of length 16000.  25 random segments of length 80 were removed from the signal.
sSPGP's time complexity ($\cO(S K^2 + K^3)$ with mini-batch of size $S$ and $K$ inducing points) prohibits it from being used with a large number of inducing points. We therefore used 800 inducing points for sSPGP and 400 inducing inputs for each component in the covariance function of sfVSSGP ($K=400$). 

\begin{figure}[b!]
\vspace{-8mm}
\hspace{-3mm}
\subfloat[Train error]{
\includegraphics[width=0.33\linewidth,trim = 2.5mm 3mm 2.5mm 2mm, clip]{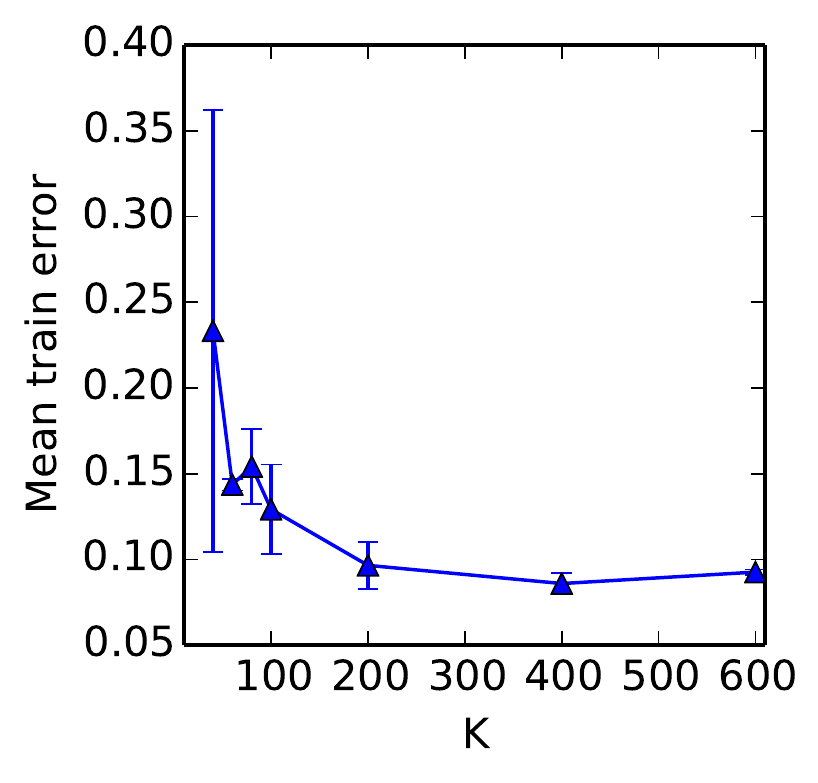}
}
\subfloat[Test error]{
\includegraphics[width=0.33\linewidth,trim = 2.5mm 3mm 2.5mm 2mm, clip]{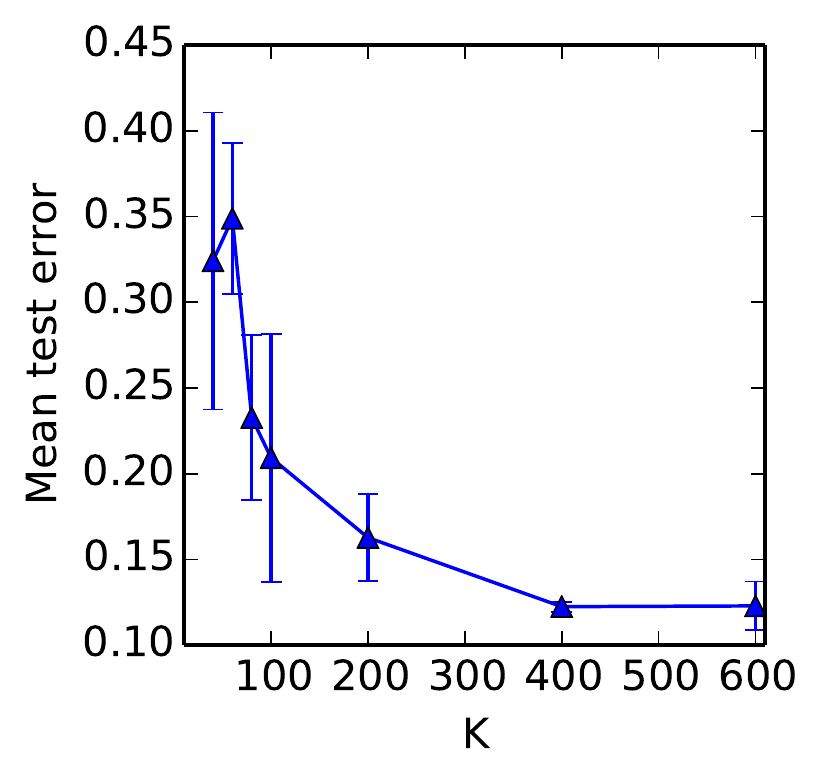}
}
\subfloat[Running time]{
\includegraphics[width=0.33\linewidth,trim = 2.5mm 3mm 2.5mm 2mm, clip]{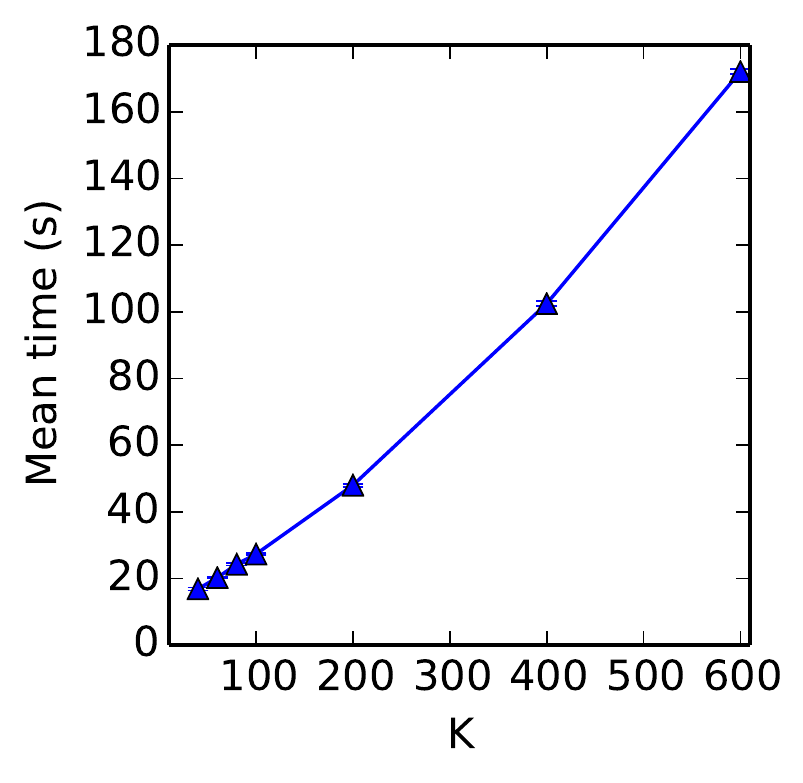}
}
\caption{Mean and standard deviation for train error, test error, and running time, all as functions of the number of inducing points ($K$) for a speech signal segment of length 4K.}
\vspace{-10mm}
\label{fig:exp3}
\end{figure}

The RMSE of sSPGP for the training set is $0.043$ and for the test set is $\textbf{0.034}$ (with a training time of 133 minutes using GPy \citep{gpy2014}). The RMSE of sfVSSGP for the training set is $0.016$ and for the test set is $\textbf{0.034}$ (with a training time of 48 minutes). Using the same audio imputation metric as in the previous section, we get that the STFT RMSE for the sSPGP on the training set is $0.54$ and on the test set is $\textbf{0.43}$. The STFT RMSE for the VSSGP on the training set is $0.18$ with a test error of $\textbf{0.3}$. 
For comparison again, baseline performance of predicting 0 attains an error of $0.52$ on the training set and an error of $\textbf{0.62}$ on the test set.

\subsection{Speed-Accuracy Trade-off}

We inspect the speed-accuracy trade-off of the approximation (RMSE as a function of the number of inducing points) for the sfVSSGP approximation. 
We repeat the same audio experiment set-up with a speech signal with 4000 samples and evaluate the imputation error (RMSE) of the predictive mean with various numbers of inducing point. RMSPROP was run for 500 iterations. The experiment was repeated 5 times and the results averaged.

Figure \ref{fig:exp3} shows that the approximation offers increased accuracy with an increasing number of inducing points. No further improvement is achieved with more than 400 inducing points. The time scales quadratically with the number of inducing points. 
Note that the approximation does not over-fit to the data as the number of parameters increases. 

\section{Discussion}
Our approximate inference relates to the Bayesian neural network \citep[\textit{Bayesian NN}, ][]{mackay1992evidence,mackay1992practical}.
In the Bayesian NN a prior distribution is placed over the weights of an NN, and a posterior distribution (over the weights and outputs) is sought. 
The model offers a Bayesian interpretation to the classic NN, with the desired property of uncertainty estimates on the outputs. 
Inference in Bayesian NNs is generally hard, and approximations to the model are often used \citep[pp 277-290]{Bishop2006Pattern}. 
Our GP approximate inference relates Bayesian NNs and GPs, and can be seen as a method for tractable variational inference in Bayesian NNs with a single hidden layer. 

Future research includes the extension of our approximation to deep GPs \citep{damianou2012deep}. We also aim to use the approximate model as a method for adding and removing units in an NN in a principled way. Lastly, we aim to replace the cosines in the Fourier expansion with alternative basis functions and study the resulting approximate model. 

\bibliography{ref}
\bibliographystyle{icml2015}

%% file: VSSGP-appendix.tex
\appendix
\section{Appendix}
\begin{identity}\label{identity:1}
\[
\cos(x-y) 
=
\int_{0}^{2\pi} \frac{1}{2\pi} 
\sqrt{2} \cos(x + b) \sqrt{2} \cos(y + b) \td b
\]
\end{identity}
\begin{proof}
We first evaluate the term inside the integral. We have
\[
&\cos(x+b)\cos(y+b) \\
&\quad = (\cos(x)\cos(b) - \sin(x)\sin(b)) \\
&\quad \qquad \cdot(\cos(y)\cos(b) - \sin(y)\sin(b)) \\
&\quad = (\cos(x)\cos(y)) \cos^2(b) + (\sin(x)\sin(y)) \sin^2(b) \\
&\quad \qquad - (\sin(x)\cos(y) + \cos(x)\sin(y)) \sin(b) \cos(b).
\]
Now, since $\int \cos^2(b) \td b = \frac{b}{2} + \frac{1}{4} \sin(2b)$, as well as $\int \sin^2(b) \td b = \frac{b}{2} - \frac{1}{4} \sin(2b)$, and $\int \sin(b) \cos(b) \td b = - \frac{1}{4} \cos(2b)$, we have
\[
&\int_{0}^{2\pi} \frac{1}{2\pi} 
\sqrt{2} \cos(x + b) \sqrt{2} \cos(y + b) \td b \\
&\quad = \frac{1}{\pi} (\cos(x) \cos(y)(\pi - 0) \\
&\quad\qquad + \sin(x)\sin(y)(\pi - 0) \\
&\quad\qquad - (\sin(x) \cos(y) + \cos(x) \sin(y)) \cdot 0 \\ 
&\quad = \cos(x-y)
\]
\end{proof}

\begin{identity}\label{identity:2}
\[
E_{\N(\w; \mu, \Sigma)} \big( \cos(\w^T \x + b) \big) =
e^{-\frac{1}{2} \x^T \Sigma \x} \cos(\mu^T \x + b)
\]
\end{identity}
\begin{proof}
We rely on the characteristic function of the Gaussian distribution to prove this identity. 
\[
&E_{\N(\w; \mu, \Sigma)} \big( \cos(\w^T \x + b) \big) \notag \\
&\quad=
\Re \bigg( e^{ib} E_{\N(\w; \mu, \Sigma)} \big( e^{i \w^T \x} \big) \bigg) \notag \\
&\quad=
\Re(e^{ib} e^{i \w^T \mu - \frac{1}{2} \x^T \Sigma \x}) \notag \\
&\quad= 
e^{-\frac{1}{2} \x^T \Sigma \x} \cos(\mu^T \x + b) 
\]
where $\Re(\cdot)$ is the real part function, and the transition from the second to the third lines uses the characteristic function of a multivariate Gaussian distribution.
\end{proof}

\begin{identity}\label{identity:3}
\[
&E_{\N(\w; \mu, \Sigma)} \big( \cos(\w^T \x + b)^2 \big) \notag\\
&\qquad \qquad \qquad = \frac{1}{2} e^{-2\x^T \Sigma \x} \cos(2\mu^T \x + 2b) 
+ \frac{1}{2}
\]
\end{identity}
\begin{proof}
Following the identity $\cos(\theta)^2 = \frac{\cos(2\theta)+1}{2}$,
\[
&E_{\N(\w; \mu, \Sigma)} \big( 
\cos(\w^T \x + b)^2 
\big) 
\notag \\ &=
\frac{1}{2} E_{\N(\w; \mu, \Sigma)} \big( 
\cos(2\w^T \x + 2b)
\big) 
+ \frac{1}{2}
\notag \\ &=
\frac{1}{2} e^{-2\x^T \Sigma \x} \cos(2\mu^T \x + 2b) 
+ \frac{1}{2}
\]
\end{proof}

\begin{proposition}\label{prop:2}
Given a sum of covariance functions with $L$ components (with each corresponding to $\Phi_i$ an $N \times K$ matrix) we have $\Phi = [\Phi_i]_{i=1}^L$ an $N \times LK$ matrix. 
\end{proposition}
\begin{proof}
We extend the derivation of equation \ref{eq:Y_given_A_X_o} to sums of covariance functions. Given a sum of covariance functions with $L$ components 
\[
K(\x, \y) = \sum_{i=1}^L \sigma_i^2 K_i(\x,\y),
\]
following equation \ref{eq:bochner}
we have
\[
K(\x, \y) = \sum_{i=1}^L \int_{\R^Q} \sigma_i^2 p_i(\w) \cos(2 \pi \w^T(\x-\y)) \td \w,
\]
where we write $\sigma_i^2$ instead of $\sigma \sigma_i^2$ for brevity (with $\sigma_i^2$ not having to sum to one).

Following the derivations of equation \ref{eq:Y_given_A_X_o}, for each component $i$ in the sum we get $\Phi_i$ an $N \times K$ matrix.
Writing $\Phi = [\Phi_i]_{i=1}^L$ an $N \times LK$ matrix, we have that the sum of covariance matrices can be expressed with a single term after marginalizing $\F$ out,
\[
\sum_{i=1}^L \Phi_i \Phi_i^T + \tau^{-1} \I = \Phi \Phi^T + \tau^{-1} \I,
\]
thus identity \ref{eq:Y_given_A_X_o} still holds. 
\end{proof}

\begin{proposition}\label{prop:3}
Performing a change of variables to the SM covariance function with a single component, results in $p(\w)$ a standard normal distribution with covariance function hyper-parameters expressed in $\Phi$.
\end{proposition}
\begin{proof}
The SM covariance function's corresponding probability measure $p(\w)$ is expressed as a mixture of Gaussians,
\[
p(\w) &= \sum_{i=1}^L \sigma_i^2 
\prod_{q=1}^Q \sqrt{2 \pi} l_{iq}
e^{-\frac{(2 \pi l_{iq})^2}{2} 
(w_q - \frac{1}{p_{iq}})^2} \notag \\
&= 
\sum_{i=1}^L \sigma_i^2 
\N( \w; \bp_i, \bL_i^{-2} ),
\]
with $\sigma_i^2$ summing to one.

Following equation \ref{eq:bochner} with the above $p(\w)$ we perform a change of variables to get,
\[
&K(\x,\y) \notag\\
&\quad = \sum_{i=1}^L 
\int_{\R^Q} \sigma_i^2 \N( \w'; \bp_i, \bL_i^{-2} ) \cos(2 \pi \w'^T(\x-\y)) \td \w' \notag\\
&\quad = \sum_{i=1}^L 
\int_{\R^Q} \sigma_i^2 \N( \w; \bz, \I ) \cos(2 \pi (\bL_i^{-1} \w + \bp_i)^T(\x-\y)) \notag\\
&\qquad\qquad\qquad\quad \cdot \td \w
\]
for $\w' = \bL_i^{-1} \w + \bp_i$.

For each component $i$ we get $\Phi_i$ an $N \times K$ matrix with elements
\[
\sqrt{\frac{2\sigma_i^2}{K}} \cos \big(2 \pi (\bL_i^{-1} \w_k + \bp_i)^T (\x - \z_k) + b_k \big),
\] 
where for simplicity, we index $\w_k$ and $b_k$ with $k=1,...,LK$ as a function of $i$.
\end{proof}

\begin{proposition}\label{prop:1}
Let $p(\ba) = \N(\bz, \I)$. The optimal distribution $q(\ba)$ solving 
\[
&\int q(\ba) \int q(\bo) \log p(\y | \ba, \X, \bo) \td \bo \td \ba \notag\\
&\qquad \qquad \qquad \qquad 
- \KL(q(\ba) || p(\ba)) - \KL(q(\bo) || p(\bo)) 
\]
is given by
\[
q(\ba_d) = \N(\bSigma E_{q(\bo)}(\Phi^T) \y_d, ~\tau^{-1} \bSigma)
\]
with $\bSigma = (E_{q(\bo)}(\Phi^T \Phi) + \tau^{-1} I)^{-1}$.

The lower bound to optimise then reduces to
\[
&\cL = 
\sum_{d=1}^D \bigg( -\frac{N}{2} \log(2 \pi \tau^{-1}) -\frac{\tau}{2} \y_d^T\y_d
\notag \\ &\qquad \qquad \quad 
+ \frac{1}{2} \log(|\tau^{-1} \bSigma|) 
\notag\\
&\qquad \qquad \quad 
+ \frac{1}{2} \tau \y_d^T E_{q(\bo)}(\Phi) \bSigma E_{q(\bo)}(\Phi^T) \y_d
\bigg)
 \notag \\ &\qquad 
- \KL(q(\bo) || p(\bo)). 
\]
\end{proposition}
\begin{proof}

Let 
\[
&\cL = \int q(\ba) \int q(\bo) \log p(\y | \ba, \X, \bo) \td \bo \td \ba \notag\\
&\qquad \qquad 
- \int q(\ba) \log \frac{q(\ba)}{p(\ba)} \td \ba - \int q(\bo) \log \frac{q(\bo)}{p(\bo)} \td \bo.
\]
We want to solve 
\[
\frac{\td (\cL + \lambda \int (\int q(\ba) \td \ba - 1)) }{\td q(\ba)} = 0
\] 
for some $\lambda$. I.e.\ 
\[
\int q(\bo) \log p(\y | \ba, \X, \bo) \td \bo - \log \frac{q(\ba)}{p(\ba)} - 1 + \lambda = 0.
\]
 This means that 
\[
q(\ba) &= e^{\lambda - 1} e^{\int q(\bo) \log p(\y | \ba, \X, \bo) \td \bo} p(\ba) \\
&= 
\exp \bigg( 
-\frac{1}{2} \ba^T \tau(E(\Phi^T \Phi) + \tau^{-1}I) \ba \\
&\qquad\qquad\qquad\qquad\qquad\qquad + \big( \tau \y^T E(\Phi) \big) \ba + ...
\bigg)
\]
and since $q(\ba)$ is Gaussian, it must be equal to
\[
q(\ba) = \N(\bSigma E_{q(\bo)}(\Phi^T) \y, ~\tau^{-1} \bSigma)
\]
with $\bSigma = (E_{q(\bo)}(\Phi^T \Phi) + \tau^{-1} I)^{-1}$.

Writing $p(\ba)$ and $q(\ba)$ explicitly and simplifying results in the required lower bound.

\end{proof}

\begin{proposition}\label{prop:4}
Denoting $\M = [\m_d]_{d=1}^D$, we have 
\[
E_{q(\y^* | \x^*)}(\y^*) = E_{q(\bo)} \big( \phi_* \big) \M.
\]
\end{proposition}
\begin{proof}
The $d$'th output $y_d^*$ of the mean of the distribution is given by (writing $\phi_* = \phi(\x^*, \bo)$)
\[
E_{q(y_d^* | \x^*)}(y_d^*)
&= \int y_d^* p(y_d^* | \x^*, \A, \bo) q(\A, \bo) \td \A \td \bo \td y_d^* \notag\\
&\quad = \int \big( \phi_* \ba_d \big) q(\A, \bo) \td \A \td \bo \notag\\
&\quad = 
\int \phi_* q(\bo) \td \bo 
\int \ba_d q(\A) \td \A \notag\\
&\quad = 
E_{q(\bo)} \big( \phi_* \big) \m_d,
\]
which can be evaluated analytically following equation \ref{eq:exp_Phi}. 
\end{proof}

\begin{proposition}\label{prop:5}
The variance of the predictive distribution is given by
\[
&\Var_{q(\y^* | \x^*)}(\y^*) 
=
\tau^{-1}\I_D + \Psi  \\ 
&\qquad + \M^T \big( E_{q(\bo)}\big(\phi_*^T \phi_*\big) - E_{q(\bo)} \big( \phi_* \big)^T E_{q(\bo)} \big( \phi_* \big)\big) \M\notag
\]
with $\Psi_{i,j} = \tr \big( E_{q(\bo)}\big(\phi_*^T \phi_*\big) \cdot \s_i \big) \cdot \mathds{1}[i=j]$.
\end{proposition}
\begin{proof}
The raw second moment of the distribution is given by (remember that $\y^*$ is a $1 \times D$ row vector)
\[
&E_{q(\y^* | \x^*)}((\y^*)^T(\y^*)) \notag\\
&\quad = \int \bigg( (\y^*)^T(\y^*) p(\y^* | \x^*, \A, \bo)  \td \y^* \bigg) q(\A, \bo) \td \A \td \bo \notag\\
&\quad = \int \big( \Cov_{p(\y^*|\x^*, \A, \bo)}(\y^*)  \notag \\
&\qquad + E_{p(\y^*|\x^*, \A, \bo)}(\y^*)^T E_{p(\y^*|\x^*, \A, \bo)}(\y^*) \big) q(\A, \bo) \td \A \td \bo \notag \\
&\quad = \tau^{-1} \I_D + E_{q(\A)q(\bo)}\big( \A^T \phi_*^T \phi_* \A \big).
\]
Now, for $i \neq j$ between $1$ and $D$, 
\[
\bigg( E_{q(\A)q(\bo)}\big( \A^T \phi_*^T \phi_* \A \big) \bigg)_{i,j} &= E_{q(\A)q(\bo)}\big( \ba_i^T \phi_*^T \phi_* \ba_j \big) \notag\\
&= \m_i^T E_{q(\bo)}\big(\phi_*^T \phi_*\big) \m_j,
\]
and for $i = j$ between $1$ and $D$, 
\[
\bigg( E_{q(\A)q(\bo)}\big( \A^T \phi_*^T \phi_* \A \big) \bigg)_{i,i} &= E_{q(\A)q(\bo)}\big( \ba_i^T \phi_*^T \phi_* \ba_i \big) \notag\\
&= 
\m_i^T E_{q(\bo)}\big(\phi_*^T \phi_*\big) \m_i\notag\\
&\qquad +
\tr \bigg( E_{q(\bo)}\big(\phi_*^T \phi_*\big) \cdot \s_i \bigg)
\]
following equation \ref{eq:exp_a_Phi_Phi_a}. 

Taking the difference between the raw second moment and the outer product of the mean we get that the variance of the predictive distribution is given by
\[
&\Var_{q(\y^* | \x^*)}(\y^*) 
=
\tau^{-1}\I_D + \Psi  \\ 
&\qquad + \M^T \big( E_{q(\bo)}\big(\phi_*^T \phi_*\big) - E_{q(\bo)} \big( \phi_* \big)^T E_{q(\bo)} \big( \phi_* \big)\big) \M\notag
\]
with $\Psi_{i,j} = \tr \big( E_{q(\bo)}\big(\phi_*^T \phi_*\big) \cdot \s_i \big) \cdot \mathds{1}[i=j]$.
\end{proof}

\begin{discussion}\label{dis:1}
We discuss some of the key properties of the VSSGP, fVSSGP, and sfVSSGP. Due to space constraints, this discussion was moved to the appendix.

Unlike the sparse pseudo-input approximation, where the variational uncertainty is over the locations of a sparse set of inducing points in the output space, the uncertainty in our approximation is over a sparse set of function frequencies. 
As the uncertainty over a frequency ($\Sigma_k$) grows, the exponential decay term in the expectation of $\Phi$ decreases, and the expected magnitude of the feature ($[(E_{q(\bo)}(\Phi))_{n,k}]_{n=1}^N$) tends to zero for points $\x_n$ far from $\z_k$. Conversely, as the uncertainty over a frequency decreases, the exponential decay term increases towards one, and the expected magnitude of the feature does not diminish for points $\x_n$ far from $\z_k$. 

With the predictive uncertainty in equation \ref{eq:pred_uncertainty} we preserve many of the GP characteristics. As an example, consider the SE covariance function\footnote{Given by $\sigma^2 \exp \big(
-\frac{1}{2} \sum_{q=1}^Q \frac{(x_q - y_q)^2}{l_{q}^2}
\big)$}. In full GPs the variance increases towards $\sigma^2 + \tau^{-1}$ far away from the data. This property is key to Bayesian optimisation for example where this uncertainty is used to decide what action to take given a GP posterior. 

With the SE covariance function, our expression for $\phi_*$ contains an exponential decay term $\exp(-\frac{1}{2} (\x_n - \z_k)^T \Sigma_k (\x_n - \z_k))$. This term tends to zero as $\x_n$ diverges from $\z_k$. For $\x_n$ far away from $\z_k$ for all $k$ we get that the entire matrix $\Phi$ tends to zero, and that $E_{q(\bo)}\big(\phi_*^T \phi_*\big)$ tends to $\frac{\sigma^2}{K} \I_k$. 

For fVSSGP, equation \ref{eq:pred_uncertainty} then collapses to 
\*[
&\Var_{q(\y^* | \x^*)}(\y^*) 
=
\tau^{-1}\I_D + \Psi'
\]
with $\Psi'_{i,j} = 
\sigma^2 \frac{1}{K} \sum_{k=1}^K (\mu_{ik}\mu_{jk} + s_{ik}^2\mathds{1}[i=j])$.

This term leads to identical predictive variance to that of the full GP when $\A$ is fixed and follows the prior. It is larger than the predictive variance of a full GP when $s_{di}^2 > 1 - \mu_{di}^2$ on average, and smaller otherwise.

Unlike the SE GP, the predictive mean in the VSSGP with a SE covariance function does not tend to zero quickly far from the data. This is because the model can have high confidence in some frequencies, driving the inducing frequency variances ($\Sigma_k$) to zero. This in turn requires $\x_{n}-\z_k$ to be much larger for the exponential decay term to tend to zero. The frequencies the model is confident about will be used far from the data as well.

Unlike the SSGP, the approximation presented here is not periodic. This is one of the theoretical limitations of the sparse spectrum approximation (although in practice the period was observed to often be larger than the range of the data). The limitation arises from the fact that the covariance is represented as a weighted sum of cosines in SSGP. In the approximation we present here this is avoided by decaying the cosines to zero.

It is interesting to note that although our approximated covariance function $K(\x, \y)$ has to be stationary (i.e.\ it can be represented as $K(\x,\y) = K(\x-\y)$), the approximate posterior is not. This is because stationarity entails that for all $\x$ it must hold that $K(\x, \x) = K(\x - \x) = K(\bz)$. But for $E_{q(\bo)}(\Kh(\X, \X)) = E_{q(\bo)}(\Phi \Phi^T)$ we have that the diagonal terms depend on $\x$:
\*[
\big( E_{q(\bo)}(\Phi \Phi^T) \big)_{n,n} = 
\sum_{k=1}^K
&\frac{2\sigma_i^2}{K} e^{-\bx_{nk}^T \Sigma_k \bx_{nk}}
\notag \\ & \cdot
E_{q(b_k)} \big( \cos(\mu_k^T \bx_{nk}) + \bb_{nk}) \big)^2.
\]
This is in comparison to the SSGP approximation, where the approximate model is stationary.

It is also interesting to note that the lower bound in equation \ref{eq:lower_bound_factorised} is equivalent to that of equation \ref{eq:lower_bound_opt_A} for $\s_d$ non-diagonal. For $\s_d$ diagonal the lower bound is looser, but offers improved time complexity. 

The use of the factorised lower bound allows us to save on the expensive computation of $\A$ for small updates of $\bo$. Intuitively, this is because small updates in $\bo$ would result in small updates to $\A$. Thus solving for $\A$ analytically at every time point without re-using previous computations is very wasteful. Optimising over $\A$ to solve the linear system of equations (given $\bo$) allows us to use optimal $\A$ from previous steps, adapting it accordingly.

Also, even though it is possible to analytically integrate over $\A$, we can't analytically integrate $\bo$. This is because $\bo$ appears inside a cosine inside an exponent in equation \ref{eq:Y_given_A_X_o}. We can't solve for $\bo$ analytically either in the equation preceding equation \ref{eq:exp_a_Phi}. This is again because $\bo$ appears inside a cosine (unlike $\A$ which appears in a quadratic form in that equation). 

Finally, we can approximate our approach to achieve a much more scalable implementation by only using the $K'$ nearest inducing inputs for each data point. This is following the observation that for short length-scales and large $\Sigma$, the features will decay to zero exponentially fast with the distance of the data points from the inducing inputs. 
\end{discussion}
